\documentclass{article}

\usepackage[utf8]{inputenc} 
\usepackage[T1]{fontenc}    
\usepackage{hyperref}       
\usepackage{url}            
\usepackage{booktabs}       
\usepackage{amsfonts,dsfont}       
\usepackage{nicefrac}       
\usepackage{microtype}      
\usepackage{graphicx}
\usepackage{color}
\usepackage{algorithmic}
\usepackage{algorithm2e}
\usepackage{subcaption}

\usepackage{amsmath,amsthm}
\newcommand{\norm}[1]{\left\lVert#1\right\rVert}
\newcommand{\J}{\mathcal{K}}
\DeclareMathOperator{\prob}{prob}
\newcommand{\ind}{\scalebox{1.15}{$\mathds{1}$}}
\newcommand{\reals}{\ensuremath{\mathbb{R}}}
\newcommand{\coloneq}{\stackrel{\textup{\tiny def}}{=}}

\def\shownotes{1}  \ifnum\shownotes=1
\newcommand{\authnote}[2]{$\ll$\textsf{\footnotesize #1 notes: #2}$\gg$}
\else
\newcommand{\authnote}[2]{}
\fi

\newcommand{\al}[0]{\alpha}

\newcommand{\rh}[0]{\rho}

\newcommand{\Te}[0]{\Theta}

\newcommand{\om}[0]{\omega}

\newcommand{\iy}[0]{\infty}

\newcommand{\pa}[1]{\left( {#1} \right)}

\newcommand{\ve}[1]{\left\Vert {#1}\right\Vert}

\newcommand{\set}[2]{\left\{{#1}:{#2}\right\}}

\newcommand{\diag}{\operatorname{diag}}

\providecommand{\cal}[1]{\mathcal{#1}}
\renewcommand{\cal}[1]{\mathcal{#1}}

\newcommand{\pull}[9]{
#1\ar@/_/[ddr]_{#2} \ar@{.>}[rd]^{#3} \ar@/^/[rrd]^{#4} & &\\
& #5\ar[r]^{#6}\ar[d]^{#8} &#7\ar[d]^{#9} \\}

\newcommand{\cmp}[9]{
\xymatrix{
#1 \ar[r]^{#4}{#5} \ar@/_2pc/[rr]^{#8}_{#9} & #2 \ar[r]^{#6}_{#7} & #3
}
}

\newcommand{\ha}[1]{\ar@{^(->}[#1]}
\newcommand{\ls}[1]{\ar@{-}[#1]}
\newcommand{\sj}[1]{\ar@{->>}[#1]}
\newcommand{\aq}[1]{\ar@{=}[#1]}
\newcommand{\acir}[1]{\ar@{}[#1]|-{\textstyle{\circlearrowright}}}
\newcommand{\acil}[1]{\ar@{}[#1]|-{\textstyle{\circlearrowleft}}}
\newcommand{\ard}[1]{\ar@{.>}[#1]}
\newcommand{\mt}[1]{\ar@{|->}[#1]}
\newcommand{\inm}[1]{\ar@{}[#1]|-{\in}}
\newcommand{\inr}{\ar@{}[d]|-{\rotatebox[origin=c]{-90}{$\in$}}}
\newcommand{\inl}{\ar@{}[u]|-{\rotatebox[origin=c]{90}{$\in$}}}

\newcommand{\beq}[1]{\begin{equation}\llabel{#1}}
\newcommand{\eeq}[0]{\end{equation}}
\newcommand{\bal}[0]{\begin{align*}}
\newcommand{\eal}[0]{\end{align*}}
\newcommand{\ban}[0]{\begin{align}}
\newcommand{\ean}[0]{\end{align}}

\newcommand{\fixme}[1]{{\color{red}#1}}
\newcommand{\llabel}[1]{\label{#1}\text{\fixme{\tiny#1}}}

\newcommand{\arxiv}[1]{\url{http://www.arxiv.org/abs/#1}}

\allowdisplaybreaks[2]

\DeclareFontFamily{U}{wncy}{}
    \DeclareFontShape{U}{wncy}{m}{n}{<->wncyr10}{}
    \DeclareSymbolFont{mcy}{U}{wncy}{m}{n}
    \DeclareMathSymbol{\Sh}{\mathord}{mcy}{"58} 

\theoremstyle{plain}
\newtheorem{theorem}{Theorem}
\newtheorem{proposition}{Proposition}
\newtheorem{assumption}{Assumption}

\title{Robust Spectral Filtering \\ and Anomaly Detection}

\author{
  Jakub Marecek and Tigran Tchrakian \\
  IBM Research -- Ireland \\
  \texttt{\{jakub.marecek,tigran\}@ie.ibm.com}
}

\begin{document}

\maketitle

\begin{abstract}
We consider a setting, where the output of a linear dynamical system (LDS) is, with an unknown but fixed probability, replaced by noise. 
There, we present a robust method for the prediction of the outputs of the LDS
and identification of the samples of noise, and prove guarantees on its statistical performance.
One application lies in anomaly detection: the samples of noise, unlikely to have been generated by our estimate of the unknown dynamics, 
can be flagged to operators of the system for further study.
\end{abstract}

\section{Introduction}

Across mathematics, statistics \cite{shumway1982approach}, artificial intelligence \cite{narendra1990identification}, and engineering \cite{ljung1998system,van2012subspace}, much attention has been devoted to the identification of linear dynamical systems (LDS):
\begin{align}
	\label{simplest}
h_k &= A h_{k-1} + B x_k + \eta_k \\
y_k &= C h_k     + D x_k + \zeta_k, \notag 
\end{align}
where $x_k \in \reals^{n}$ are inputs, $y_k \in \reals^{m}$ are outputs, $h_k \in \reals^{d}$ is a hidden (latent) state,
$A,B,C,D$ are compatible matrices, and $\eta_k, \zeta_k$ are compatible noise vectors with $\sum_{k=1}^T \norm{\eta_k}^2 + \norm{\zeta_k}^2<L$.
In improper learning of such an LDS (which we refer to as an identification problem), one wishes to estimate $\hat y_k$ 
 such that $\hat y_k$ are close to the best estimates $y^*_k$ of $y_k$
 possible at time $k$.
When there is no hidden state, the identification problem is convex and a variety
of methods work well.
When there is a hidden state, the problem is non-convex and only rather recently
 spectral filtering \cite{hazan2017online,hazan2018spectral} has been used to obtain 
identification procedures with regret bounded by $\tilde{O}(\log^7 \sqrt{k} )$ at time $k$,
where 
$\tilde O(\cdot)$ hides terms that depend polynomially on the dimension of the system and
 norms of the inputs and outputs and the noise.

We consider a Huber-like \cite{huber1981robust} setting, where 
with a fixed probability $p > 0$, which may be known or unknown,
 the  observations $y_k$ are replaced by noise.
That is, we have:
\begin{align}
	\label{robust}
h_k &= A h_{k-1} + B x_k + \eta_k \\
y_k &= \begin{cases}
\xi_k & \textrm{with probability } p \\
C h_k     + D x_k + \zeta_k, & \textrm{otherwise}
\end{cases} \notag 
\end{align}
under assumptions described out in the next section. 
Our goals are two-fold: first, to predict $\hat y_{k+1}$ of
 $C h_{k+1} + D x_{k+1} + \zeta_{k+1}$.
Our second goal is to identify when $\xi_k$ corrupts the observations such that it can be flagged for further study by operators of the system in the spirit of anomaly detection.
We stress that this Huber-like model differs from the settings for both additive and non-additive changes surveyed in~\cite{Basseville1993}, where the additive changes are the changes in the mean of the distribution of the observed signals and non-additive changes are related to changes in variance, correlations, spectral characteristics, or dynamics of the signal or system. 

Overall, our contributions are as follows:
\begin{itemize}
	\item We present a novel Huber-like model for anomaly detection.
	\item We present algorithms combining spectral filtering and additive-decrease multiplicative-increasing (ADMI) for the related robust identification problem.
	\item We present conditions that allow for the identification of the corrupting noise and the underlying LDS.
\end{itemize}
Notice that our analytical results for the Huber-like model are stronger than those surveyed in~\cite{Basseville1993} in three ways.
First, we suggest what is the absolute value of the difference between samples of $\xi_k$ and the non-corrupted observation $C h_k     + D x_k + \zeta_k$
sufficient to detect that an anomaly occurred at time $k$.
Second, we provide guarantees on the regret of our estimate of the subsequent observation of $C h_k     + D x_k + \zeta_k$, under the conditions, 
where anomalies are detectable.
Third, in contrast to the usual assumption of a Gaussian process noise $\eta_k$ and masurement noise $\zeta_k$, 
we allow for arbitrarily-distributed, but bounded noise
$\sum_{k=1}^T \norm{\eta_k}^2 + \norm{\zeta_k}^2<L$,
or equivalently, 
a bounded amount of adversarial perturbations to the system.
We hence believe that our model and the results make for a valuable addition to the literature on anomaly
 detection based on 
non-additive changes.

\section{The Problem}

As has been suggested in the previous section, we consider the problem of predicting $\hat y_{k+1}$ in 
the Huber-like extension of LDS \eqref{robust},
 under several assumptions, starting with the identifiability of Hazan et al. \cite{hazan2018spectral}:

\begin{assumption}
	\label{ass1}
	The outputs are generated by the stochastic difference equation \eqref{robust}, assuming:
\begin{enumerate}
\item Inputs and outputs are bounded: $\norm{x_t}_2 \leq R_x, \norm{y_t}_2 \leq R_y$.
\item The system is Lyapunov stable, i.e., the largest singular value of $A$ is at most 1: $\rh(A) \leq 1$. 
\item $A$ is diagonalizable by a matrix with small entries: $A = \Psi \Lambda \Psi^{-1}$, with $\ve{\Psi}_F\ve{\Psi^{-1}}_F \le R_\Psi$. 
\item $B,C,D$ have bounded spectral norms: $\ve{B}_2,\ve{C}_2,\ve{D}_2 \le R_\Te$.
\item Let $S=\set{\al/|\al|}{\al\text{ is an eigenvalue of }A}$ be the set of phases of all eigenvalues of $A$. There exists a monic polynomial $p(x)$ of degree $\tau$ such that $p(\om)=0$ for all $\om\in S$, the $L^1$ norm of its coefficients is at most $R_1$, and the $L^\iy$ norm is at most $R_\iy$. 
\end{enumerate}
\end{assumption}

While the Assumption \ref{ass1} may seem restrictive, it essentially says that the 
system is identifiable \cite{1101375} and that eigenvectors of $A_i$ corresponding to 
larger eigenvalues are not linearly dependent.
Indeed, since Kalman \cite{kalman1963mathematical}, it is understood that 
from input-output measurements, only the part of the system that is controllable and observable
can be identified, 
while one can clearly achieve \cite{simchowitz2018learning} a near-perfect prediction $\hat y$ of the output for an unstable system.
Further, it is clear \cite{trefethen2013approximation} that the polynomial $p(x)$ does exist, 
and we only introduce the notation for the norms of its coefficients.

We also make assumptions concerning the sparse noise, i.e., distribution of $\xi_k$ and probability $p$. 
Ideally, one would like to consider:

\begin{assumption}
	\label{ass2}
Probability $p < 1$ is not known and the noise $\xi_k$ is arbitrarily distributed.
\end{assumption}


We discuss Assumption \ref{ass2} in Section \ref{sec:anal}. Notice, however, that some samples 
of the arbitrarily distributed noise may be indistinguishable from the output of the 
linear dynamical system. 
For a strong result on the identification of the LDS, 
we consider a separation condition:

\begin{assumption}
	\label{assDk}
Probability $p < 1$ is not known.
At time $k$, the absolute value of the difference between the noise $\xi_k$ and $C h_k + D x_k + \zeta_k$ is greater than
some instance- and algorithm-specific $D_k$.
\end{assumption}

In Section \ref{sec:anal}, we prove the existence of $D_k$, which goes to 0 in the large limit of $k$.
Under this assumption, we can also estimate the unknown $p$.

\section{The Algorithms}

At a high-level, we suggest to use Algorithm~\ref{alg:schema} under Assumption \ref{ass1}.
At the current time $t$, the algorithm has the history of inputs, $x_1 \ldots x_t$, and the outputs, $y_1 \ldots y_{t-1}$, available. 
Based on the current input, $x_t$, the algorithm produces a forecast, $\hat y_t$, after which the output, $y_t$, of the (possibly corrupted) real system is observed.
We then test whether the loss $\norm{y_t - \hat y_t}^2$
is less than a certain threshold $D_t$. 
If it is, we assume that $y_t$ was generated by the LDS
and use it in further predictions.
Otherwise, we assume that the value is a sample of $\xi_t$
and do not use it for further predictions.  
In the next section, Proposition \ref{HazanLike} shows that
there is a $D_t$ of Assumption \ref{assDk} decreasing to 0 in the large limit of 
$t$ at a rate of $\tilde O(t^{-1/2}\log^7(t))$,
which makes this schema meaningful.

In particular, we consider Algorithm~\ref{alg:AIMD} based on additive-decrease multiplicative-increase (ADMI),
under Assumptions \ref{alg:schema} and \ref{assDk}.
There, we consider $D_t$ of the form
$\textrm{mean}(L_k) + c \textrm{std}(L_k)$,
where $L_k$ are losses for predictions of values generated by the LDS,
mean is the arithmetic mean, 
std is the standard deviation,
and $c$ is a coefficient greater or equal to $1.0$. 
Subsequently, we update $c$ using $\alpha > 0, \beta > 1.0$ as follows:
\begin{align}
	c \leftarrow & \begin{cases} \beta c      & \textrm{ if } l_t > D_t \\
                          c - \alpha          & \textrm{ otherwise. } \end{cases}
\end{align}
That is: when we detect an anomaly, we raise the threshold for detecting anomalies
relative to the losses observed so far.
Otherwise, we decrease the threshold
relative to the losses observed so far.
Similar policies are widely used \cite{corless2016aimd} for congestion management in TCP/IP networking and distributed 
resource allocation. 
Again, we analyse this approach in the next section, specifically in Theorem~\ref{IFS1}.

Throughout both Algorithms~\ref{alg:schema} and \ref{alg:AIMD},
we predict the next output $\hat y_t$ of the system 
from inputs $X_t$ until time $t$ and outputs $Y_{t-1}$ until $t-1$
in an online fashion.
There, leading methods \cite{hazan2017online,hazan2018spectral,6426006,hardt2016gradient,simchowitz2018learning}
consider an overparametrisation, where the vector $\tilde X_t$ is composed of 
the inputs to the system at all time-levels up to the current one, convolved with the eigenvectors of a certain Hankel matrix, as well as the outputs at the previous time level, and inputs at the current and previous time levels. 
Notice that the Hankel matrix is constant and its eigenvectors can be precomputed.
See 
\cite[e.g.]{Verhaegen1992,Overschee1994,Fazel2001,Liu2010,Smith2012} for background, 
\cite{hazan2018spectral} for the detailed derivation of the method we use, 
and Algorithm \ref{alg:AIMD} for a sketch of our implementation. 
We note that the new hypothesis class $\hat{\cal{H}}$ arising from the over-parametrization \cite{hazan2018spectral}
has been shown by \cite{hazan2018spectral} to approximately contain the class of LDS 
satisfying Assumption \ref{ass1}, which makes it possible to derive regret bounds considering the convexification:
\begin{align}
\label{convexification}
f(M) = \sum\limits_{i=1}^t\norm{ y_i - M \tilde X_i}^2, \textrm{ where } M \in \hat{\cal{H}},
\end{align}
instead of the non-convex problem at each point in time. In theory, one has to consider
the convexifications growing with $T$, but in practice, windowing works well. 
Furthermore, one can apply on-line optimisation techniques, such as a small number
 of iterations of a coordinate descent between two time levels,
which benefit from the facts that the problem is strongly convex and
that the optimizer of \eqref{convexification} changes only modestly between two time levels.

\begin{algorithm}[t]
\caption{A schema of an algorithm for the setting of Assumption \ref{assDk}.}
\label{alg:schema}
\begin{algorithmic}[1]
\item[] \textbf{Input}: time horizon $T$ 
\STATE Initialize $M_t$ to suitable dimension. 
\FOR{$t = 1, \ldots, T$}
\STATE Form overparametrisation $\tilde X_t$ from inputs $X_t$ until time $t$ and outputs $Y_{t-1}$ until $t-1$ 
\STATE $\hat y_t \gets M_t \tilde X$  
\STATE Observe $y_t$ and compute $l_t := \norm{y_t - \hat y_t}^2$
\IF{$l_t > D_t$}
\STATE Update $M$ for next time-level
\STATE Update $D_{t+1}$, if needed
\STATE Consider $x_t, y_t$ for subsequent time levels
\ELSE 
\STATE Ignore $x_t, y_t$ for subsequent time levels
\STATE Update $D_{t+1}$, if needed
\ENDIF
\ENDFOR
\STATE Return $\hat y_T$
\end{algorithmic}
\end{algorithm}


\begin{algorithm}[t]
\caption{An algorithm for the setting of Assumption \ref{assDk} based on additive-decrease multiplicative-increase (ADMI).}
\label{alg:AIMD}
\begin{algorithmic}[1]
\item[] \textbf{Input}: time horizon $T$, data points $X_T, Y_{T-1}$, number $k$ of filters, pre-computed top $k$ eigenpairs  $\{(\sigma_j, \phi_j)\}_{j=1}^k$ of a certain matrix $Z_T$ 
\item[] \textbf{Output}: prediction $\hat y_{T}, \hat p$
\STATE Initialize $M_t$ to suitable dimension, initialise $L_t$ to empty list, initialise $c$ to 1, $e_0$ to 0
\FOR{$t = 1, \ldots, T$}
\STATE Form overparametrisation $\tilde X_t$ from inputs $X_t$ until time $t$, outputs $Y_{t-1}$ until $t-1$, and convolutions with 
pre-computed $\{(\sigma_j, \phi_j)\}_{j=1}^k$, as in \cite{hazan2018spectral} 
\STATE $\hat y_t \gets M_t \tilde X$ 
\STATE Observe $y_t$ and compute $l_t := \norm{y_t - \hat y_t}^2$
\STATE Set $D_t := \textrm{mean}(L_t) + c \textrm{std}(L_t)$
\IF{$l_t > D_t \textbf{ or } \textrm{ extra}(l_t, D_t, \hat p_t)$}
\STATE 
\label{line:optimisation}
Update $M$ for next time-level: $M_{t+1} \gets \arg\min\limits_{M} \sum\limits_{i=1}^t\norm{ y_i - M \tilde X_i}^2$  
\STATE Add $l_t$ to $L_t$
\STATE Update $c$ to $\beta c$
\STATE Set $e_t$ to $e_{t-1} + 1$
\ELSE
\STATE Update $c$ to $c - \alpha$
\STATE Set $e_t$ to $e_{t-1}$
\ENDIF
\STATE Estimate $\hat p_t$ as $e_t / T$
\ENDFOR
\STATE Predict and return $\hat y_{T}$ and $\hat p$
\end{algorithmic}
\end{algorithm}

\section{An Analysis}
\label{sec:anal}

We could start with a result based on the work of Hazan et al. \cite{hazan2018spectral}:

\begin{proposition}
\label{HazanLike}
Under Assumption \ref{ass1}, Algorithm~\ref{alg:schema} 
makes it possible to consider $D_k$ of Assumption \ref{assDk} decreasing to 0 in the large limit of $k$ at a rate of $\tilde O(k^{-1/2}\log^7(k))$,
where $\tilde O(\cdot)$ hides dependence on instance-specific constants.
Furthermore, this choice of $D_k$ allows for the 
perfect recovery of the probability $p$ of \eqref{robust}.
Furthermore, this allows for the perfect recovery of the non-corrupted entries.
\end{proposition}

\begin{proof}
We want to show that:
	\begin{align}
	\sum_{k}	\| C h_k     + D x_k + \zeta_k - \xi_k \| \le \tilde O\pa{R_1^3 R_x^2 R_\Theta^4 R_\Psi^2 R_y^2 d^{5/2} n \log^7 k \sqrt{k}} + 
	O(R_\iy^2 \tau^{3} R_\Te^2 R_\Psi^2 L),
	\end{align}
	where the $\tilde O(\cdot)$ suppresses factors polylogarithmic in $n,m,d,R_\Theta,R_x,R_y$ 
	allows for the perfect recovery of the non-corrupted entries.  
	Recall that the parameters $R_1$, $R_x$, $R_\Theta$, $R_\Psi$, and $R_y$ are those in Assumption~\ref{ass1}, 
	$d$ is the dimension of the hidden state space, 
	$n$ is the dimension of the input space,
	and $k$ is the number of time steps. 
This follows directly from Theorem 19 of Hazan et al. \cite{hazan2018spectral}.
\end{proof}

Notice that this result is limited in two ways: 
First, it is not constructive, because the instance-dependent terms, based on the constants in Assumption \ref{ass1}, 
such as bounds on the spectral norms of matrices, are unknown \emph{a priori} and non-trivial to estimate on-line.
Second, one may wish for the width of the interval to scale with (the square root of the second moment of the) losses obtained so far, because
in many practical situations, the actual losses may be less than our analytical upper bound thereupon (and the second moment is important in the confidence estimates). 
To address these issues, we consider a policy, which dynamically adapts $D_k$ based on the additive-decrease multiplicative-increase 
(ADMI) updates:

\begin{theorem}
\label{IFS1}
Under Assumption \ref{ass1}, Algorithm~\ref{alg:AIMD} even with $\textrm{extra}(L_t, D_t, \hat p_t)$ 
being a constant function returning False, 
makes it possible to compute $D_t$ of Assumption \ref{assDk} scaling linearly with the mean of the losses observed so far.
Furthermore, this choice of $D_t$ allows for 
our estimate $\hat p_t$ of the probability $p$ \eqref{robust} to 
converge in distribution as $t\rightarrow\infty$.
\end{theorem}

Theorem~\ref{IFS1} says that there is a distribution such that, as $t\rightarrow\infty$, $\hat p_t$ follows this distribution. 
In other words, $\hat p_t$ will not be totally erratic, but it will be predictable in the sense that it will eventually resemble
samples from a random variable with a fixed distribution. 
This legitimizes the use of simulations and studying the resulting sample distributions.
Furthermore, under mild but technical conditions, the convergence occurs at a geometric 
rate \cite[Theorem 1]{steinsaltz1999}, i.e.,
$\mathbb{E} \left[ \hat p_t - p \right]$ is $O(r^n)$, where the growth rate $r < 1$ 
can be made explicit by a careful analysis of the drift function.
Under further technical conditions, one could prove moment bounds  \cite{walkden2007invariance}.

\begin{proof}[Proof sketch] 
The process $\hat p_t$ can be cast as a recurrent iterated function system (RIFS) on the normed space $(\reals,\norm{\cdot}_1)$ with a family of functions $\bigl\{\omega_j \mid j\in\J
\bigr\}$ that take one of the following two forms:
\begin{align}
	\omega_1(c) &:= \beta c     & \quad \beta > 1.0 \\
	\omega_2(c) &:= c - \alpha  & \quad \alpha > 0
\end{align}
where clearly only $\omega_2$ is a contraction. 
We provide an overview of RIFS in Appendix \ref{sec:ifs} in the Supplementary material. 
By using a theorem of Barnsley et al. (restated as Theorem \ref{thm:markch} in the Supplementary material, for convenience), we want to show that the system converges in distribution.

In particular, the probability of applying $\omega_1$ at iteration $t+1$ is given by the probability of:
$l_k > D_k$ 
where $l_k$ is the loss $\norm{y_t - \hat y_t}^2$
and $D_k$ is a threshold. 
Our goal is hence to prove that there exists a Markov chain, with $K$ states and a transition probability matrix $P\in [0,1]^{K\times K}$, such that $\prob(i_{t+1}=j|i_t) = p_{i_t j}$, i.e., the probability of applying a specific $\omega_j$ depends on  the last applied function $\omega_{i_t}$.
Clearly, this is not true in case of $K = 2$, but one can consider an abitrary number of copies of the two functions,
and hence a much larger $K$.

Notice that even with the hidden state, the evolution of 
of the underlying $y_t$ in \eqref{simplest} can be easily modelled with a Markov chain.
Our goal is hence to show that the evolution of $\hat y_t$
can be modelled by a Markov chain. 
Although the solution of the convexification
\eqref{convexification} in Line \ref{line:optimisation} may seem non-linear, one should consider the
fact that 
the existence of a Moore–Penrose pseudoinverse guarantees that 
there exists a linear representation of the evolution of $y_t$.
The existence of a such a Markovian representation in turn guarantees that 
there exist instance-specific conditions 
such that the RIFS is contractive on average and by Theorem~\ref{thm:markch}, we then conclude that $\hat p_k$ converges in distribution.  




\end{proof}

Notice, however, that so far, 
we have not considered the setting of Assumption~\ref{ass2} and 
we have not made any use of the $\hat p_t$.
Let us now consider Assumptions \ref{ass1} and \ref{ass2} and a function
%
$\textrm{extra}(L_t, D_t, \hat p_t),$ 
which would from some time level $t_0$ onwards perform two actions:
First, consider the $t_0$ most recent entries in $L_t, D_t$ 
and compute an estimate of a probability of an anomaly based on such 
time-window of length $t_0$.
Second, estimate the probability that such an estimate is \emph{not} 
drawn from a sample distribution of $\hat p_t$.
This is motivated by the intuition that in case $l_t$ is less than $D_t$,
we may want to force a $\hat p_t$ fraction of any time-window to be an anomaly.
We conjecture that under Assumptions \ref{ass1} and \ref{ass2}, 
Algorithm~\ref{alg:AIMD} 
with such an extra
allows for the recovery of the non-corrupted entries,
with high probability in the large limits of $t, t_0$.
Intuitively, the proof may 
use the technique of the proof of Theorem~\ref{IFS1}
and a law of large numbers for non-identically 
distributed Bernoulli random variables, e.g., from Kolmogorov's strong law.
However, the proof would have to operate with a much larger state space of the Markov chain 
than the one used in the proof of Theorem~\ref{IFS1}, and the reasoning would be complicated by the
fact that until $t_0$, we may corrupt our estimate of the underlying LDS by 
mis-interpreting some elements of noise $\xi_t$ as outputs of the LDS, 
which Theorem~\ref{IFS1} avoids by Assumption \ref{assDk}. 

\section{Empirical Results}

To illustrate the performance of the algorithms, we chose the same 
single-input single-output (SISO)
system  as in~\cite{hazan2018spectral}, where:
\begin{align}
	\label{HazanEx}
B^\top = C = [1 \, 1], \quad D = 0, \quad A = \diag([0.999,0.5]),
\end{align}
time horizon $T=100$, and noise terms $\eta$ and $\xi$ are i.i.d. Gaussians. 
For comparison purposes, we consider the trivial last-value prediction,
also known as persistence-based prediction, which uses the most recent
value $\hat y_{t+1} := y_t$,
and the same thresholding of Algorithm \ref{alg:AIMD}.

First, we illustrate the performance on one sample run of the method.
Figure \ref{fig:io} presents the true inputs $x_t$ (in blue), true outputs $y_t$ (in green), 
which have been corrupted in 10\% of samples by noise $\mathcal{U}(0, 100)$
(as indicated by black vertical bars in the bottom third of the picture),
and predictions $\hat y_t$ of the output of our method (in red).
Below, Figure \ref{fig:events} presents the thresholds (in dotted lines),
the output of our method (in blue), the last-value prediction (in green, often overlapping with the red line),
and the corresponding anomalies as semi-transparent vertical bars. 
In particular, the pink vertical bars in the top third of the plot
correspond to anomalies detected by the last-value prediction and the
pale blue bars in the middle third of the plot correspond to the 
thi anomalies detected by our method.
In this one sample, with no statistical significance, the
harmonic mean of precision and recall (F1 score) of our method is 0.93,
while the last-value prediction results in F1 score of 0.56.
Notice that the uniformly-distributed noise 
$\mathcal{U}(0, 100)$ violates Assumption \ref{assDk}, and
indeed, the first detected anomaly does not differ from the previous
output of the LDS by much.

In Figure \ref{fig2}, we present  the 
loss $l_t := \norm{y_t - \hat y_t}^2$ of our method (in red) and the last-value prediction (in blue).
In particular, we plot the mean (in a solid line) and standard deviation (shaded).
While it is not possible to infer any generalisations from this one
 particular LDS, the F1 score of 0.88 of our method (averaged over the sample paths)
 improves considerably over the F1 score of 0.46 using the last-value predictions.

Next, we present perhaps an even more intriguing example, which is not supported by our theory.
In particular, we consider a time-varying system, where we vary $B_k$ from $k = 50$ 
such that each entry of $B_k$ follows a sinusoid:
\begin{align}
\begin{cases}
1 & \textrm{ if } k < 50, \\
1.01 + \sin( \frac{\pi (k-50)}{180}) & \textrm{ otherwise. }
\end{cases}	
	\label{timeVaryingEx}
\end{align}
Figures \ref{figtv1} and \ref{figtv2} clearly demonstrate that the performance of the
last-value prediction does not change materially, but 
that the performance of our method improves.
We have been able to replicate this behaviour on a variety of small examples.
This naturally opens the question as to whether one could prove regret bounds for
spectral filtering in a time-varying system, and consequently guarantee the performance
of the anomaly detection therein.


\begin{figure}[t!]
    \centering
    \begin{subfigure}[b]{0.49\textwidth}
        \includegraphics[page=5, width=\textwidth]{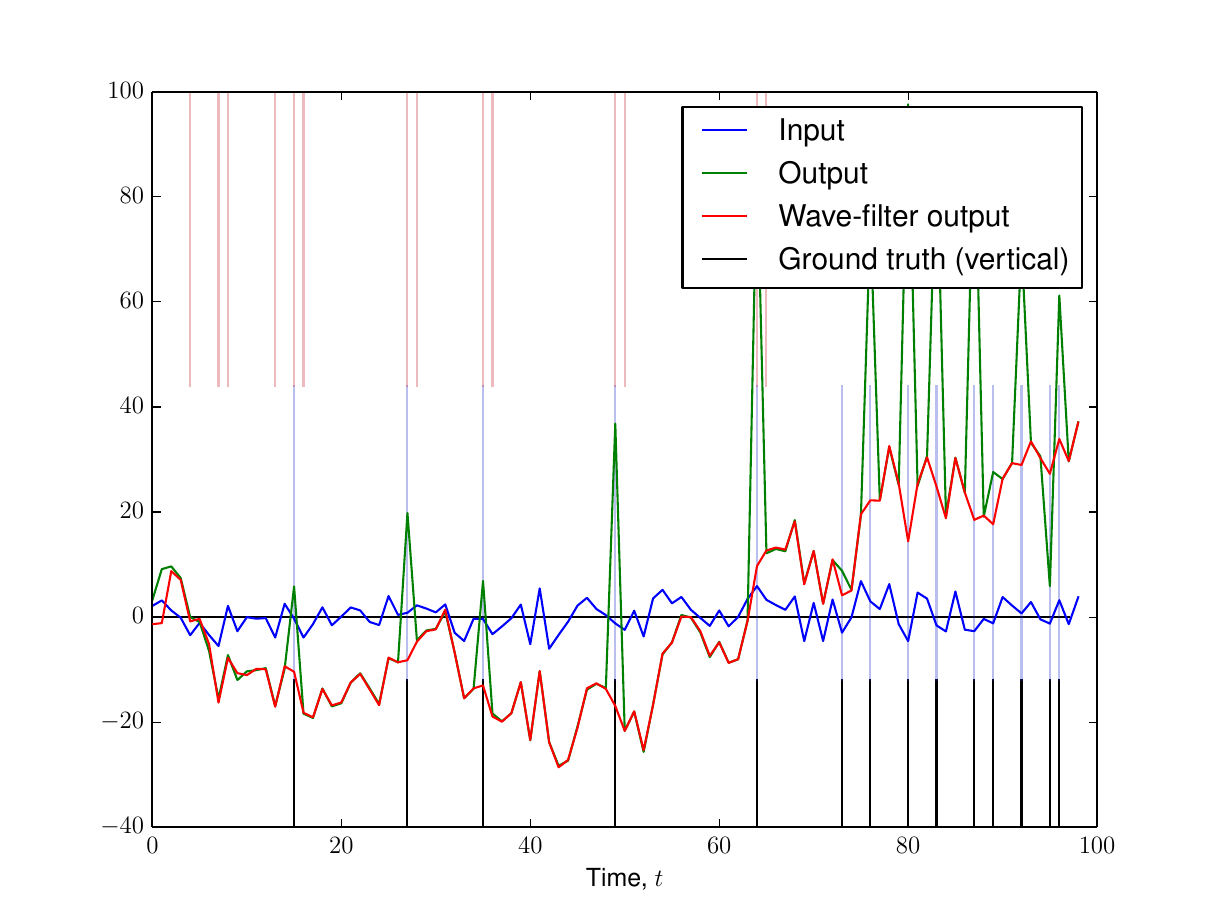}
        \caption{Inputs, outputs, and predictions of the output by our method.}
        \label{fig:io}
    \end{subfigure}
    \begin{subfigure}[b]{0.49\textwidth}
        \includegraphics[page=6,width=\textwidth]{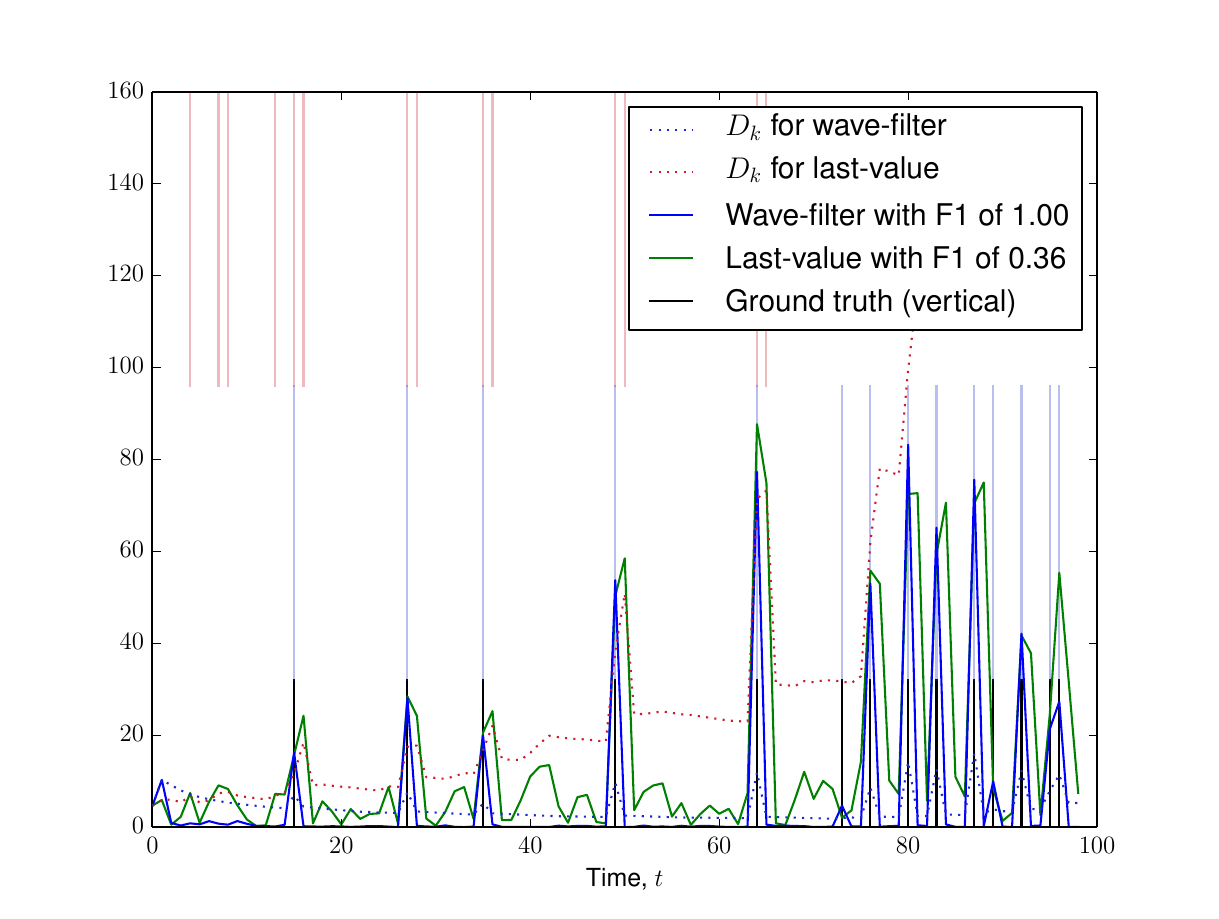}
        \caption{Thresholds and anomalies detected.\newline}
        \label{fig:events}
    \end{subfigure}
    \caption{First illustrations on \eqref{HazanEx}.}\label{fig1}
\end{figure}

\begin{figure}[t!]
    \centering
     \includegraphics[page=11, width=0.49\textwidth]{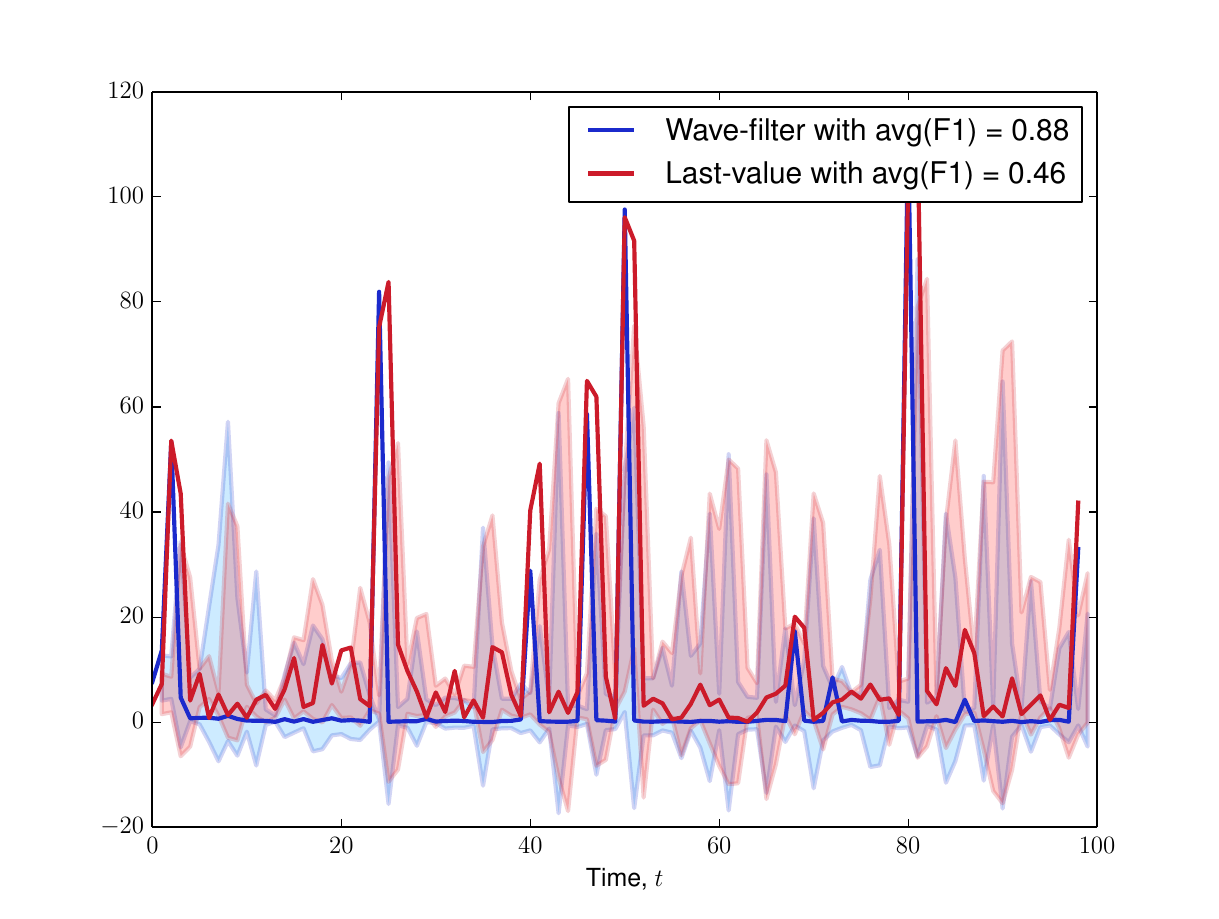}
    \caption{Mean and standard deviation of the loss $l_t$ on \eqref{HazanEx}.}\label{fig2}
\end{figure}

\begin{figure}[t!]
    \centering
    \includegraphics[page=5, width=0.49\textwidth]{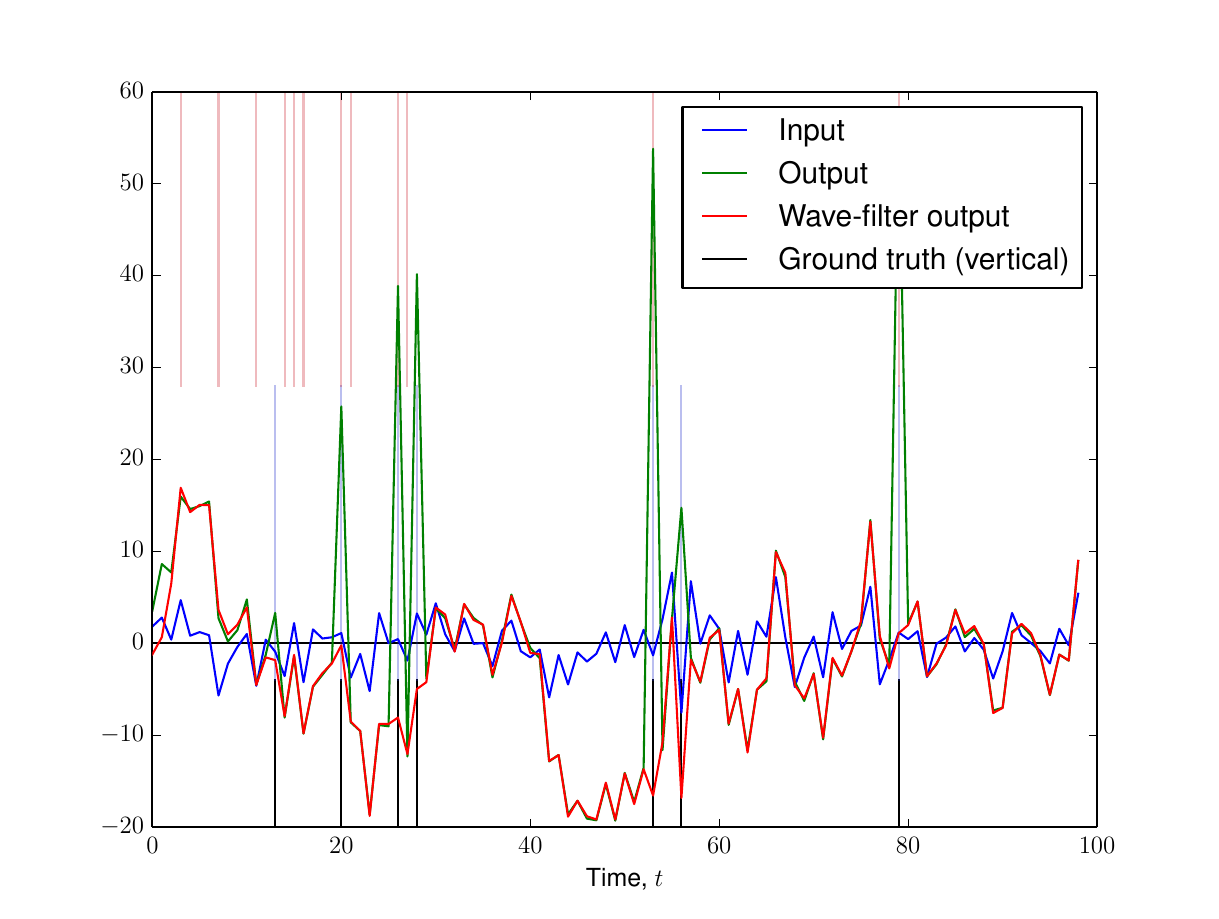}
     \includegraphics[page=6, width=0.49\textwidth]{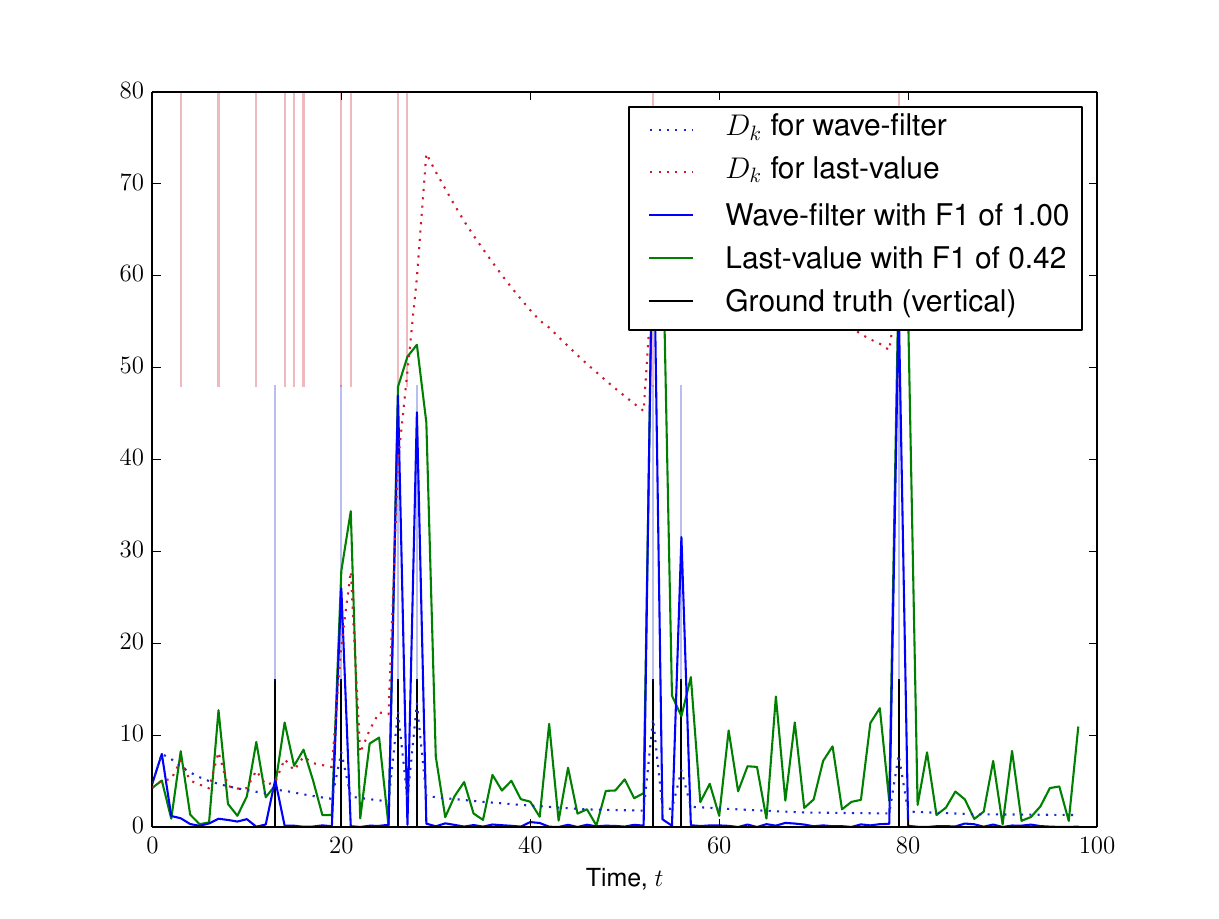}
    \caption{Illustrations on the time-varying system \eqref{timeVaryingEx}. Left: Inputs, outputs, and predictions of the output by  our method. Right: Thresholds and anomalies detected.}
     \label{figtv1}
\end{figure}

\begin{figure}[t!]
    \centering
     \includegraphics[page=11, width=0.49\textwidth]{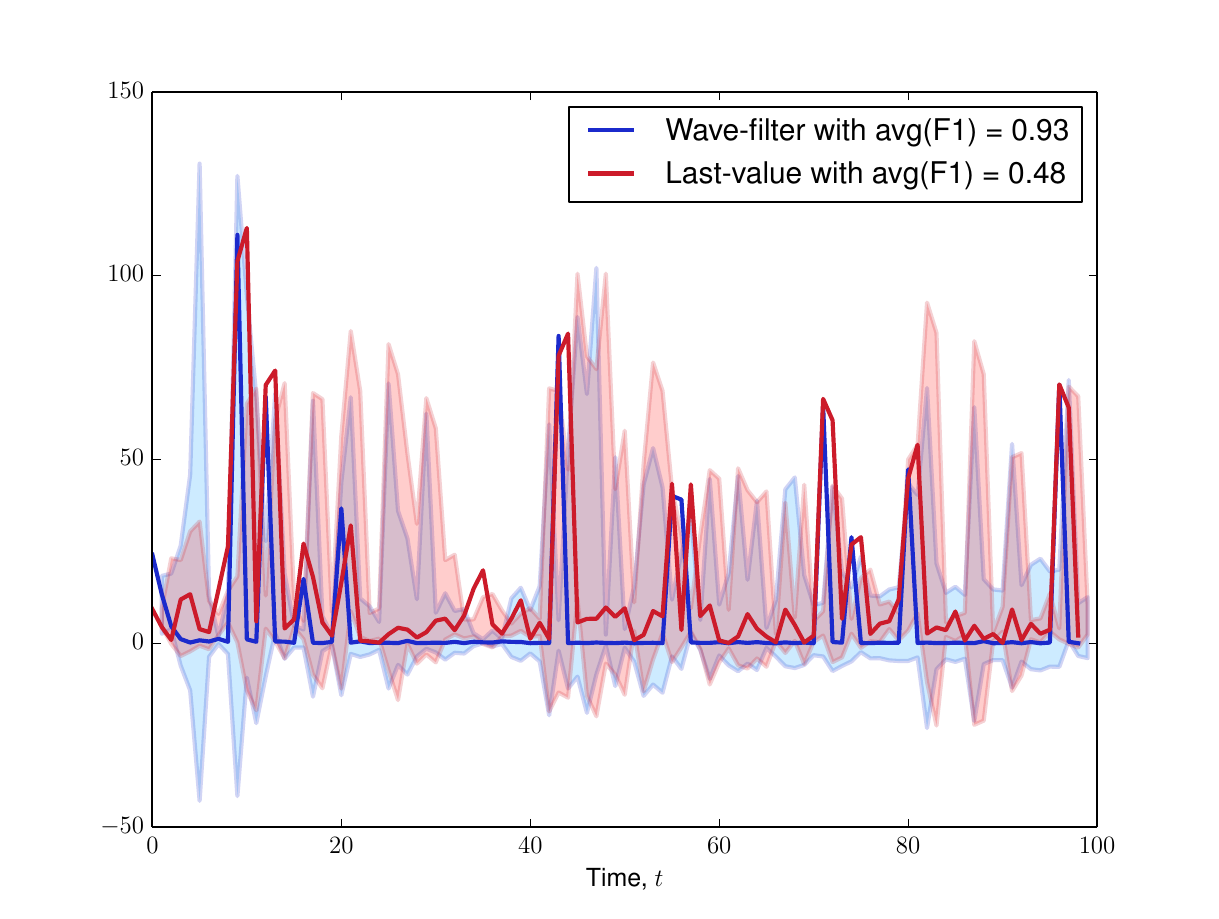}
    \caption{Mean and standard deviation of the loss $l_t$ on the time-varying system \eqref{timeVaryingEx}.}\label{figtv2}
\end{figure}

\section{Related Work}

In system identification, there are over 65 years of research \cite{kalman1963mathematical}, 
which we draw upon. 
Specifically, our spectral filtering follows the tradition of subspace methods for identification \cite[e.g.]{Verhaegen1992,Overschee1994,Fazel2001,Liu2010,Smith2012}. 
There, one consider ``wave filters'', which are based on convolving 
data with eigenvectors of a certain Hankel matrix $Z_T$.
Notice that the eigendecomposition of the Hankel matrix $Z_T$ can be pre-computed,
as the matrix does not depend on the input data.
In particular, we consider the regularised version of 
Hazan et al. \cite{hazan2017online,hazan2018spectral}. 
Several other authors \cite{6426006,hardt2016gradient,simchowitz2018learning}
have derived similarly important results at the same time. 
Subsequently, a number of authors  \cite{dean2017sample,abbasi2018regret,fazel2018global,feinberg2018model,arora2018towards,boczar2018finite}
have applied them to the (Linear-Quadratic, LQ) control of an unknown system, 
which underlies much of reinforcement learning. 

In anomaly detection \cite{chandola2009anomaly} and and closely related problems,
there is a possibly even longer history of related work. 
For anomaly detection in general, and especially for anomaly detection in LDS, the book of Basseville and Nikiforov \cite{Basseville1993} is the standard reference.
In particular, the closest to our work is Ting et al. \cite{ting2007kalman},
who employ Kalman filters \cite{kalman1963mathematical} in anomaly detection, i.e.,
 assume that $A, B$ are known. 
In contrast, we do not assume that $A, B$ are known.
Further, we should like to point to the closely related problems of deviation \cite{palpanas2003distributed}, or (on-line, complex) event \cite{yang1998study} detection, 
    outlier analysis \cite{aggarwal2015outlier}, detection \cite{subramaniam2006online}, or pursuit \cite{xu2010robust},  
    foreground detection \cite{BOUWMANS2014} and the complementary background subtraction or background maintenance, 
    or even dynamic anomalography \cite{Mardani2013}.

\section{Conclusions}

While anomaly detection is notoriously hard to benchmark, due to the fact that each application
has its own assumptions as to what is normal, we believe that the assumption of normal data being 
generated by an unknown linear dynamical system 
and anomalies replacing the observations arbitrarily 
\eqref{simplest} in a Huber-like fashion may have a broad appeal,
especially in conjunction with the methods with performance guarantees, 
which we have presented. 

There is a considerable scope for further work, including the rates of convergence
 \cite[cf. Theorem 1]{steinsaltz1999}, moment bounds  \cite[cf.]{walkden2007invariance},
 and extensions of the results to time-varying systems on the theoretical side, 
 and novel variants of the thresholding on the algorithmic side.
In particular, we envision that exponential smoothing and upper confidence bounds
 may well be worth investigating. 

\small
\bibliography{references,pursuit}
\bibliographystyle{plain}

\clearpage
\appendix

\section{Background on Iterated Function Systems}
\label{sec:ifs}
In a generalization of a Markov chain, known 
variously as  \textit{iterated function system} or  \textit{iterated random functions},
one has a state space $X$ with its metric $d$, a family $W$ of Lipschitz functions\footnote{%
	A function $f$ on the metric space $(X,d)$ is Lipschitz with constant $s$, or ``$s$-Lipschitz,'' if for all $x,y\in X$, we have $d\bigl(f(x),f(y)\bigr) \leq s d(x,y)$.
} $W=\{w_j:X\rightarrow X \, | \, j\in \J\}$, where $\J$ is some index set,
which we assume to be finite or countably infinite,
and a measure $\nu$ that makes $(\J,\cdot,\nu)$ a probability space. 

At each iteration $t$ of the iterated function system, $j$ is selected from $\J$ according to $\nu$ and $w_j$ is applied to the current state $x_k$ to obtain $x_{k+1}$. Formally:
\[
\prob(X_{k+1}\in A | X_k=x_k) \coloneq \sum_{\J} \ind_{\{i | w_i(x_k)\in A  \}}(j) \nu(j),
\]
i.e.\ the probability of $x_{k+1}$ ending up in a set $A$ is the probability of selecting an index $j$ such that $w_j(x_k)$ is in $A$ (the measure of the set of indices $j$ for which $w_j(x_k)$ is in $A$). Here the Markov property is clear: the distribution of the next state $X_{k+1}$ depends only on the current state $x_k$ and not any ``older'' states $x_{k-1}$ etc.

In this way, the IFS ``jumps'' around $X$. Unless we have a degenerate case such as all $w_j$ having the same fixed point, we can not expect the sequence $\{x_k\}_{k=0}^{\infty}$ to converge in a classical sense.
Instead, we can establish conditions for \textit{convergence in distribution}: that there is a distribution $\Pi$ on $X$ such that as $k\rightarrow\infty$, the set $\{x_0,x_1,\dotsc,x_k\}$ will be distributed according to $\Pi$.

\begin{theorem}[E.g. \protect{\cite[Thm. 1.1]{DiaconisFreedman1999}}] \label{thm:ifs}
	Let $L_j$ denote the Lipschitz constant of $w_j$ and assume that the IFS is \emph{contractive on average}, i.e.
	\begin{equation}
	\sum_{\J} \nu(j) \log(L_j)<0. \label{eq:contronavg}
	\end{equation}
	Then, there is a distribution $\Pi$ on $X$ such that $\{x_0,x_1,\dotsc,x_k\}$ is distributed according to $\Pi$ as $k\rightarrow\infty$.
\end{theorem}
If $W$ is a family of contractions, i.e.\ if $L_j<1$ for all $j$, then \eqref{eq:contronavg} is trivially satisfied.

One can generalise this notion further \cite{BarnsleyEltonHardin1989} to a \textit{recurrent iterated function system} (RIFS), which is an IFS with an underlying Markov chain that modifies $\nu$ at each time step. More precisely, we have an IFS as described in the last section with a finite index set $\J$, say $\J=\{1,\dotsc,K\}$. Additionally, there is a Markov Chain with $K$ states and transition probability matrix $P\in [0,1]^{K\times K}$. The probability of applying $w_j$ at iteration $k+1$ is now given by  $\prob(i_{k+1}=j|i_k) = p_{i_k j}$, i.e.\ the probability of applying a specific $w_j$ depends on what the last applied function $w_{i_k}$ was! This is in contrast to the basic case above, where the probability to select a specific $w_j$ was always the same and given by $\nu(j)$. Notice that the way that $X_k$ jumps around in $X$ now is \emph{not} a Markov process anymore -- the distribution of $X_k$ not only depends on $X_{k-1}$, but also on $i_{k-1}$ --- but the joint process of $(X_k,i_k)$ jumping around in $X\times \J$ is.

Results analogous to Theorem~\ref{thm:ifs} can be stated for this case, see e.g.\ \cite{barnsley1988invariant,BarnsleyEltonHardin1989}. We state
\begin{theorem}[\cite{BarnsleyEltonHardin1989}]
	\label{thm:markch}
	Assume we have an RIFS as described above, and let $m:\{1,\dots,K\}\rightarrow[0,1]$ denote the stationary distribution of the underlying Markov chain (i.e.\ $m$ corresponds to the normalized Perron eigenvector of $P^T$). Then, if
	\begin{equation}\label{eq:avgcontractive}
		\sum_{i=1}^K m(i) \log L_i = E_m\{ \log L_{i}  \} <0,
	\end{equation}
	there is a unique stationary distribution $\tilde{\nu}$ of the Markov process $(X_k,i_k)$ and $X_k$ converges in distribution to $\nu$ with $\nu(B) = \tilde{\nu}(B\times\J)$. Here, $L_i$ again denotes the Lipschitz constant of $w_i$, and $E_m$ denotes expected value with respect to $m$. 
\end{theorem}
\begin{proof}
	This is just a corollary (much weaker, but sufficient for our purposes) to \protect{\cite[Thm.\ 2.1 (ii)]{BarnsleyEltonHardin1989}}, which follows by taking $n=1$ and removing the specifics of the stationary distributions.\qed
\end{proof}

If~\eqref{eq:avgcontractive} holds, we again say that the RIFS is \textit{average contractive} or \textit{contractive on average}.


\end{document}